\documentclass[lettersize,journal]{IEEEtran}
\usepackage{cite}
\usepackage{amsmath,amssymb,amsfonts}
\usepackage{algorithmic}
\usepackage{graphicx}
\usepackage[ruled]{algorithm2e}
\usepackage{hyperref}
\hypersetup{hidelinks=true}
\usepackage{textcomp}
\usepackage{subcaption}
\usepackage{multirow}
\usepackage[table,xcdraw]{xcolor}

\newtheorem{theorem}{\textbf{Theorem}}

\newtheorem{lemma}{\textbf{Lemma}}

\newtheorem{example}{Example}
\newtheorem{definition}{\textbf{Definition}}

\def\BibTeX{{\rm B\kern-.05em{\sc i\kern-.025em b}\kern-.08em
    T\kern-.1667em\lower.7ex\hbox{E}\kern-.125emX}}
\usepackage{balance}
\begin{document}
\title{GINO-Q: Learning an Asymptotically Optimal Index Policy\\ for Restless Multi-armed Bandits}
\author{Gongpu~Chen,
	Soung Chang Liew,
	Deniz Gunduz
	
	\thanks{Gongpu Chen and Deniz Gunduz are with the Department of Electrical and Electronic Engineering, Imperial College London, London, UK  (e-mail: \{gongpu.chen, d.gunduz\}@imperial.ac.uk).}
	\thanks{Soung Chang Liew is with the Department of Information Engineering, The Chinese University of Hong Kong, Shatin, Hong Kong (e-mail: soung@ie.cuhk.edu.hk).}
	
}

\markboth{Journal of \LaTeX\ Class Files}%
{How to Use the IEEEtran \LaTeX \ Templates}

\maketitle

\begin{abstract}
The restless multi-armed bandit (RMAB) framework is a popular model with applications across a wide variety of fields. However, its solution is hindered by the exponentially growing state space (with respect to the number of arms) and the combinatorial action space, making traditional reinforcement learning methods infeasible for large-scale instances. In this paper, we propose GINO-Q, a three-timescale stochastic approximation algorithm designed to learn an asymptotically optimal index policy for RMABs. GINO-Q mitigates the curse of dimensionality by decomposing the RMAB into a series of subproblems, each with the same dimension as a single arm, ensuring that complexity increases linearly with the number of arms. Unlike recently developed Whittle-index-based algorithms, GINO-Q does not require RMABs to be indexable, enhancing its flexibility and applicability. Our experimental results demonstrate that GINO-Q consistently learns near-optimal policies, even for non-indexable RMABs where Whittle-index-based algorithms perform poorly, and it converges significantly faster than existing baselines.
\end{abstract}


\section{Introduction}
\IEEEPARstart{A} restless multi-armed bandit (RMAB) models a sequential decision-making problem, in which a set of resources must be allocated to $N$ out of $M$ ($1\le N<M$) arms at each discrete time step. Here, each arm represents a dynamic process.
Upon the resource allocation, each arm generates a reward and may transition to a new state. 
These arms evolve independently except for simultaneously being subject to the resource constraint. Hence, RMABs essentially constitute a special case of weakly coupled dynamic programs \cite{adelman2008WCDP}.
The objective is to determine an optimal policy for selecting the arms at each time step in a way that maximizes the expected reward over the entire time horizon.
RMABs find applications in diverse fields, including resource allocation \cite{wang2019whittle}, opportunistic scheduling \cite{wang2021restless}, public health interventions \cite{mate2022field}, and many more.

This paper investigates the design of an efficient reinforcement learning (RL) algorithm for RMABs.
The curse of dimensionality is a central challenge in solving dynamic programs. Unfortunately, this problem is particularly severe in RMABs due to the exponential growth of RMAB’s dimension with the number of arms. In fact, it is known that RMABs are P-SPACE hard even when full system knowledge is available \cite{RMAB_PSPACEhard}. As a result, directly treating RMABs as Markov decision processes (MDPs) and applying RL algorithms is inefficient, and, in many cases, computationally infeasible—particularly for large-scale RMABs. For instance, consider a moderate-scale RMAB with $M=100$ and $N=25$. Suppose each arm exhibits 10 states. Then the RMAB encompasses an astronomical \(10^{100}\) possible states, along with $100\choose 25$ \(\approx 2.4 \times 10^{23}\) valid actions. Such complexity presents a great challenge for conventional RL algorithms.

To address this challenge, recent studies have drawn inspiration from the well-known Whittle index policy \cite{RMAB_Whittle1988}—a planning algorithm for RMABs under full system knowledge. Specifically, the Whittle index policy computes an index for each state of each arm based solely on the arm’s parameters, and then selects the $N$ arms with the highest $N$ indices at each time. This approach decouples the RMAB across arms during the index computation, significantly enhancing computational efficiency. Motivated by this structure, a series of recent studies has focused on applying RL algorithms to learn Whittle indices in settings where system knowledge is unavailable \cite{Fu2019Q4WI,avrachenkov2022whittle,xiong2023finite,nakhleh2021neurwin}.

Despite its appeal, the Whittle index policy is fundamentally limited by its reliance on the \textit{indexability} property---a condition that does not naturally hold for all RMABs. While prior studies have proposed sufficient conditions under which RMABs are indexable \cite{nino_2001,nino2007dynamic,Whittle_app2006}, these conditions are often stringent and are satisfied only by a limited class of RMABs.
In general, determining the indexability of an RMAB is a difficult task that requires a rigorous analytical effort reliant on the knowledge of the system parameters \cite{Gongpu2021TIT,liu2010indexability,villar2016indexability}. Consequently, verifying the indexability is usually unachievable in practice when full knowledge of the system is unavailable. 
As will be shown in the experiments section, applying the Whittle index policy to a non-indexable RMAB can result in arbitrarily poor performance outcomes. This underscores a significant limitation of Whittle-index-based learning algorithms: without indexability guarantees, their application can be unreliable and potentially detrimental.

In this paper, we propose a reinforcement learning algorithm for general RMABs that does not require the indexability.
Our method is inspired by a recently developed planning algorithm for RMABs known as \textit{the gain index policy} \cite{TIT2023}. In a nutshell, the gain index policy computes a gain index for each arm at each time step based on the arm's current state. Then it selects the top $N$ arms with the highest gain indices. The gain index, defined in terms of the Q function of an MDP associated with each arm, quantifies the advantage of selecting an arm in its current state compared to not selecting it. Our algorithm, termed Gain-Index-Oriented Q (GINO-Q) learning, aims to efficiently learn the gain index in model-free settings.

The gain index policy has been shown to achieve asymptotic optimality as the number of arms grows large, given a fixed selection ratio \cite{TIT2023}. However, computing gain indices remains computationally demanding—even when full system parameters are known. To address this, we propose GINO-Q learning, a three-timescale stochastic approximation algorithm designed to efficiently approximate the gain index policy in model-free settings. Our approach integrates Q-learning, SARSA, and stochastic gradient descent, operating on distinct timescales to effectively learn gain indices without relying on prior system knowledge. Specifically, we begin by decomposing the RMAB across arms through a relaxation of the hard constraint, applying the Lagrangian multiplier method to formulate $M$ unconstrained single-arm subproblems. For each subproblem, Q-learning is used to estimate the Q-function, while SARSA is employed to approximate the gradient of the average reward with respect to the Lagrange multiplier. The Lagrange multiplier itself is then updated via stochastic gradient descent. By appropriately designing the learning rates for the three updates, the algorithm naturally operates on three distinct timescales, enabling efficient and stable learning of the gain index policy.

The key advantages of GINO-Q are twofold: 
\begin{itemize}
    \item [1.] \textit{Scalability}: By decomposing the RMAB into a collection of subproblems, GINO-Q only needs to solve problems with the same dimension as a single arm. This decomposition circumvents the exponential growth of the joint state and action spaces, ensuring that computational complexity scales linearly with the number of arms $M$. As a result, GINO-Q achieves strong performance even in large-scale RMABs, where conventional RL algorithms are computationally infeasible.
    \item [2.] \textit{Applicability}: GINO-Q does not require the RMAB to be indexable, thereby significantly broadening its applicability. Our experimental results show that Whittle-index-based learning algorithms can perform poorly in non-indexable settings. In contrast, GINO-Q consistently learns near-optimal policies—even in cases where indexability does not hold.
\end{itemize}
We evaluate the performance of GINO-Q through extensive experiments, which show that it consistently outperforms existing algorithms across all tested settings.

The rest of the paper is organized as follows. Section \ref{sec:related work} discusses the related work. Section \ref{sec: formulation} states the problem formulation and preliminaries. Section \ref{sec: GINO-Q} presents the proposed GINO-Q algorithm. Section \ref{sec: discussion} provides a discussion on the robustness of GINO-Q. Section \ref{sec: exps} demonstrates the experimental results. Finally, Section \ref{sec: conclusion} concludes this paper.

\section{Related Work} \label{sec:related work}
The success of deep RL has attracted considerable research interest in its application to practical problems that can be modeled as RMABs. Notable examples include wireless sensor scheduling \cite{leong2020deep}, dynamic multichannel access \cite{wang2018deep, Burak2018}, intelligent building management \cite{wei2017deep}. However, a common limitation across these studies is the poor scalability of deep RL in RMABs. The experimental results presented in these papers are typically restricted to small-scale scenarios, with the number of arms limited to $M < 10$.

In the planning scenario, where the system knowledge is known, the Whittle index policy \cite{RMAB_Whittle1988} is recognized as one of the most efficient heuristic algorithms for addressing RMAB problems. This has inspired a series of studies focused on applying RL algorithms to learn Whittle indices. 
For example, the work in \cite{Fu2019Q4WI} studied a Q-learning algorithm to learn the Whittle indices; however, their experiments revealed that the proposed algorithm struggles to accurately learn the Whittle indices. Subsequently, Avrachenkov and Borkar \cite{avrachenkov2022whittle} introduced Whittle-Index-Based Q-learning (WIBQ), a two-timescale stochastic approximation algorithm designed to learn the Whittle indices. They proved theoretically that WIBQ converges to the Whittle indices for indexable RMABs. Further advancing this work, Xiong and Li \cite{xiong2023finite} enhanced WIBQ to handle arms with large state spaces by coupling WIBQ with neural network function approximation. Additionally, Nakhleh et al. \cite{nakhleh2021neurwin} converted the computation of Whittle indices to an optimal control problem and proposed Neurwin, a neural network-based approach for computing Whittle indices. 

In practical applications, Whittle-index-based learning algorithms have been employed in adaptive video streaming \cite{xiong2022index}, wireless edge caching \cite{xiong2024whittle}, and preventive healthcare \cite{biswas2021learn}. Moreover, the Whittle index policy has also inspired studies on online learning for RMABs \cite{wang2023optimistic, xiong2022learning, xiong2022reinforcement}.

All of these studies rely on the assumption that the underlying RMAB is indexable, making the application of the Whittle index policy viable. However, as we will show through a concrete example, not all RMABs satisfy the indexability condition.  To address scheduling problems without indexability guarantees, a gain index policy that does not require indexability was recently proposed in \cite{TIT2023}. The authors introduced a gradient-based approach to compute the gain indices. To the best of our knowledge, our proposed GINO-Q algorithm is the first RL method designed to learn the gain index policy in a model-free setting.

\paragraph*{Notation}
For any positive integer $M$, we will use $[M]$ to denote the set of positive integers between 1 and $M$, i.e., $[M]=\{1,2,\cdots,M\}$. $\mathbb{E}[\cdot]$ denotes the expectation.

\section{Problem Statement and Preliminaries} \label{sec: formulation}
An RMAB consists of $M$ arms $\{\mathcal{B}_i:i\in [M] \}$. Each arm $\mathcal{B}_i$ is an MDP represented by a tuple $(\mathcal{S}_i, \mathcal{A}_i, r_i, p_i)$, where $\mathcal{S}_i$ is the state space, $\mathcal{A}_i$ is the action space, $r_i:\mathcal{S}_i \times \mathcal{A}_i \to \mathbb{R}$ is the reward function, and $p_i$ is the transition kernel. Let $s_i^t$ and $a^t_i$ denote the state and action of $\mathcal{B}_i$ at time $t$. In particular, we have $\mathcal{A}_i = \{0,1\}$ for all $i\in [M]$. We will say that arm $\mathcal{B}_i$ is activated at time $t$ if $a^t_i=1$. At any time step, all arms evolve independently based on their actions; that is,
\begin{align*}
	&P(s^{t+1}_1,s^{t+1}_2,\cdots, s^{t+1}_M|s^{t}_1,s^{t}_2,\cdots, s^{t}_M, a^{t}_1,a^{t}_2,\cdots, a^{t}_M) \\
	=& \prod_{i=1}^{M}p_i(s^{t+1}_i|s^t_i, a^t_i).
\end{align*}
However, there is a constraint on the actions that weakly couples all the arms. Specifically, at each time step, only $N$ arms can be activated ($1\le N<M$). That is, $\sum_{i=1}^{M} a^t_i = N$ for all $t$.
The objective is to identify an optimal policy that maximizes the cumulative reward obtained from all the arms. Mathematically, an RMAB is a constrained optimization problem defined as follows:
\begin{align} \label{eq:RMABob}
	\max_{\{a^t_i:i\in [M], t\ge 1\}} \ & \lim\limits_{T\to \infty} \frac{1}{T}\mathbb{E} \left[ \sum_{t=1}^{T} \sum_{i=1}^{M}r_i\left(s^t_i, a^t_i\right) \right] \\  \label{eq:RMABst}
	\text{subject to} \ & \sum_{i=1}^{M} a^t_i = N,\quad \forall t.
\end{align}
Throughout this paper, we assume that all arms are unichain MDPs. As a result, the objective function \eqref{eq:RMABob} is always well-defined and independent of the initial state.

\subsection{The Gain Index Policy}
An RMAB  can be formulated as an MDP with a joint state space $\mathcal{S}_1\times\mathcal{S}_2\times\cdots\times\mathcal{S}_M$ and action space $\{0,1\}^M$. However, the size of the state spaces grows exponentially with $M$, and the set of feasible actions forms a combinatorial space of size $M\choose N$. As a result, the problem becomes challenging to solve when $M$ is large. Since the $M$ arms are weakly coupled only via constraint \eqref{eq:RMABst}, a common approach to mitigate the complexity is to decompose the RMAB into single-arm problems by relaxing the constraint. Building on this concept, \cite{TIT2023} introduced a gain index policy and proved its asymptotic optimality. 

In particular, we can relax constraint \eqref{eq:RMABst} of the RMAB problem to the following:
\begin{align} \label{eq: RMABrelst}
	\lim\limits_{T\to \infty}\frac{1}{T} \mathbb{E} \left[ \sum_{t=1}^{T} \sum_{i=1}^{M}a^t_i \right]  = N.
\end{align}
Replacing \eqref{eq:RMABst} by \eqref{eq: RMABrelst} leads to a relaxed RMAB. We can further transform the relaxed problem using the Lagrange multiplier method and yield an unconstrained problem:
\begin{align} \label{eq:max-min}
	\max_{\{a^t_i\} } \inf_{\lambda } \ \lim\limits_{T\to \infty}\frac{1}{T} \mathbb{E} \left[ \sum_{t=1}^{T} \sum_{i=1}^{M} \left[r_i\left(s^t_i, a^t_i\right) - \lambda a^t_i  \right] \right] + {N\lambda},
\end{align}
where $\lambda\in \mathbb{R}$ is the Lagrange multiplier. The max and inf in \eqref{eq:max-min} can be interchanged. That is, \eqref{eq:max-min} is equivalent to:
\begin{align} \label{eq:min-max}
	\inf_{\lambda} \max_{\{a^t_i\} }  \  \lim\limits_{T\to \infty}\frac{1}{T} \mathbb{E} \left[\sum_{t=1}^{T} \sum_{i=1}^{M} \left[r_i\left(s^t_i, a^t_i\right) - \lambda a^t_i  \right] \right] + {N\lambda}.
\end{align}
For any fixed $\lambda$, \eqref{eq:min-max} can be decoupled into $M$ subproblems:
\begin{align}
	J_i(\lambda): \max_{\{a^t_i:t\ge 1\} }  \ \lim\limits_{T\to \infty}\frac{1}{T} \mathbb{E} \left[  \sum_{t=1}^{T} \left[r_i\left(s^t_i, a^t_i\right) - \lambda a^t_i  \right] \right] , 
\end{align}
where $i\in [M]$. We refer to each $J_i(\lambda)$ as a \textit{single-arm problem}. Note that $J_i(\lambda)$ is an MDP associated with $\mathcal{B}_i$: they share the same state space $\mathcal{S}_i$, action space $\mathcal{A}_i$, and transition kernel $p_i$.  However, there is a distinction between them in terms of their reward functions. The reward function of $J_i(\lambda)$ is defined as $r_i(s,a)-\lambda a$, where $\lambda$ can be interpreted as the cost of action $a=1$. If we denote the optimal value of problem $J_i(\lambda)$ by $g_i(\lambda)$, it can be determined by the Bellman equation as follows \cite{puterman2014markov}:
\begin{align}
	V_i(s,\lambda) + g_i(\lambda) = \max_{a\in \{0,1\}} \left\{ Q_i(s,a, \lambda)  \right\} ,\ s\in \mathcal{S}_i,
\end{align}
where $V_i(s,\lambda)$ is the state value function and $Q_i(s,a,\lambda)$ is the state-action value function:
\begin{align*}
	Q_i(s,a,\lambda) \triangleq r_i(s,a) - \lambda a + \sum_{s'\in \mathcal{S}_i}p_i(s'|s,a)V_i(s',\lambda).
\end{align*}
Here, we express $V_i$ and $Q_i$ as functions of $\lambda$ to highlight their dependence on $\lambda$. Now, \eqref{eq:min-max} reduces to
\begin{align}  \label{eq: inf-lambda}
	\inf_{\lambda} \ f(\lambda)\triangleq \sum_{i=1}^{M} g_i(\lambda) + N \lambda.
\end{align}
Denoting by $\lambda^*$ the optimal solution to problem \eqref{eq: inf-lambda},  then gain index policy is defined as follows:
\begin{definition}[Gain index policy]
	For each arm $i\in [M]$ and each state $s\in \mathcal{S}_i$, a gain index is defined as:
	\begin{align}
		W_i(s) \triangleq Q_i(s,1, \lambda^*) -  Q_i(s,0, \lambda^*).
	\end{align}
	Then the gain index of the $i$-th arm at time $t$ is given by $W_i(s^t_i)$. The gain index policy activates the $N$ arms with the largest $N$ gain indices, with ties broken arbitrarily.
\end{definition}

Intuitively, the gain index $W_i(s)$ evaluates the ``gain" of activating arm $\mathcal{B}_i$ when it is in state $s$. It turns out that this heuristic algorithm is asymptotically optimal in $M$ under some mild conditions.
Two arms $\mathcal{B}_i$ and $\mathcal{B}_j$ are considered to belong to the same class if they share the same MDP model (i.e., having the same state space, reward function, and transition kernel). In this paper, we assume that the $M$ arms can be grouped into $K$ classes, with the proportion of the arms in the $k$-th class denoted by $\eta_k\in (0,1]$, where $\sum_{k=1}^{K}\eta_k=1$. The following result is proved in \cite{TIT2023}:

\begin{theorem}
	If $K<\infty$ is fixed, and $\{\eta_k \}$ and $N/M$ are fixed. Then the gain index policy is asymptotically optimal in the following sense:
	\begin{align*}
		\lim_{M\to \infty} \frac{1}{M} G_M^{ind} = \lim_{M\to \infty} \frac{1}{M} G_M^{opt},
	\end{align*}
	where $G_M^{ind}$ and $G_M^{opt}$ denote the cumulative reward of the RMAB under the gain index policy and the optimal policy, respectively.
\end{theorem}

\subsection{Whittle Index Policy and Indexability}
In addition to the gain index policy, the renowned Whittle index policy \cite{RMAB_Whittle1988} is also an effective approach for RMABs. It shares the same asymptotic optimality as described in Theorem 1 \cite{weber1990index}. Similar to the gain index policy, the Whittle index policy assigns an index to each state of each arm and selects at each time the top $N$ arms with the highest indices for activation. Whittle's index for state $s$ of arm $\mathcal{B}_i$ is defined as the infimum value of $\lambda$ that ensures equal optimality between taking action 0 and action 1 in state $s$ within the single-arm problem $J_i(\lambda)$. That is, 
\begin{align} \label{eq:WI}
    \text{WI}_i(s) = \inf \{\lambda: Q_i(s,1,\lambda)=Q_i(s,0,\lambda)\}.
\end{align}

A limitation of the Whittle index policy is its applicability only to indexable RMABs. Specifically, we denote the set of states in which action 0 is optimal for problem $J_i(\lambda)$ by $\mathcal{E}_i(\lambda)$. 
\begin{definition} [Indexability]
    An arm $\mathcal{B}_i$ is considered \textit{indexable} if $\mathcal{E}_i(\lambda)$ expands monotonically from the empty set to the entire state space $\mathcal{S}_i$ as $\lambda$ increases from $-\infty$ to $\infty$. An RMAB is indexable if all its arms are indexable. 
\end{definition}

Recall that $\lambda$ can be interpreted as the cost of taking action $a=1$. If an arm $\mathcal{B}_i$ is indexable, then the optimal action for this arm in state $s$ is $a=1$ if the cost $\lambda\le \text{WI}_i(s)$; otherwise, the optimal action is $a=0$. Therefore, $\text{WI}_i(s)$ can be interpreted as the ``value'' of taking action $a=1$ in state $s$. The Whittle index policy then uses this value as a heuristic to prioritize arms, selecting those with the highest indices. However, in the absence of indexability, this property no longer holds—rendering $\text{WI}_i(s)$ an invalid metric for prioritization. As a result, the Whittle index policy is well-defined only for indexable RMABs.

To illustrate that the indexability does not hold for all RMABs, we construct a non-indexable arm as a counterexample.
\begin{example} [A Non-indexable Arm]
    Consider an arm $\mathcal{B}_i$ with 6 states and transition probabilities illustrated in Fig. \ref{fig:nonWI}. The reward function is defined as $r_i(1,1)=-10, r_i(1,0)=-4, r_i(2,1)=r_i(2,0) = 4$ and $r_i(s,1)=0, r_i(s,0)=-2$ for $s\in \{3,4,5,6\}$. It can be verified that $1\in \mathcal{E}_i(\lambda)$ for $-4\le \lambda \le 2$ and $1\notin \mathcal{E}_i(\lambda)$ for $\lambda<-4$ and $\lambda > 2$. Hence the set $\mathcal{E}_i(\lambda)$ does not expand monotonically as $\lambda$ increases, implying that the arm is not indexable.
\end{example}

\begin{figure}[t]
	\centering
	\includegraphics[width=2.3in]{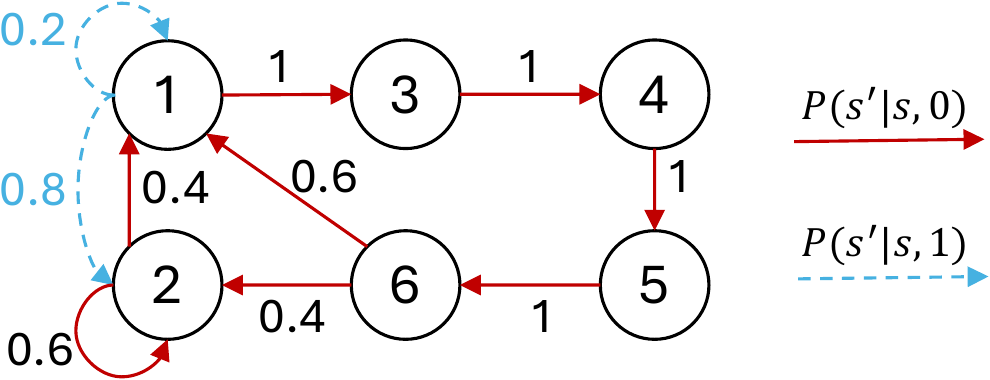}
	\caption{Transition probabilities of a nonindexable arm. Except for state 1, each of the remaining states has identical transition probabilities for both actions. For simplicity, we only plot the transitions of action 0 for those states.}
	\label{fig:nonWI}	
\end{figure}

Nevertheless, the quantity defined by \eqref{eq:WI} remains well-defined, even for non-indexable arms. This raises a natural question: \textit{what happens if the Whittle index policy is applied to a non-indexable RMAB?} Moreover, in existing Whittle-index-based algorithms such as WIBQ \cite{avrachenkov2022whittle} and Neurwin \cite{nakhleh2021neurwin}, the Whittle index is typically determined by learning a value of $\lambda$ that satisfies the condition $Q_i(s,1,\lambda) = Q_i(s,0,\lambda)$. However, in Example 1, both $-4$ and $2$ satisfy this condition. As a result, the algorithm may converge to one of these values arbitrarily, or even oscillate if multiple such solutions are close in value. \textit{How do Whittle-index-based algorithms perform under such ambiguity?}

As demonstrated by our experiments (see Section \ref{subsec: non-index}), applying the Whittle index policy to an RMAB with arms as defined in Example 1 can lead to arbitrarily poor performance.  Moreover, as previously discussed, establishing indexability is analytically challenging and typically requires system knowledge—often unavailable in learning settings. These limitations underscore the fragility of Whittle-index-based learning algorithms in the absence of indexability guarantees and highlight the importance of the gain index policy, which operates without such assumptions. These considerations motivate the development of the GINO-Q learning algorithm, which aims to learn the gain index policy efficiently in model-free settings.

\section{Gain-Index-Oriented Q Learning} \label{sec: GINO-Q}
With asymptotic optimality and no reliance on indexability, learning the gain index policy offers a promising solution for RMABs without system knowledge, particularly for large-scale problems. 
This section presents the GINO-Q learning algorithm for RMABs. Our approach involves decomposing the RMAB into  single-arm problems and learning the gain indices for each arm. 

By definition, the gain indices of an arm are determined by the Q-function of the corresponding single-arm problem under the optimal activation cost $\lambda^*$. While learning the Q function for a fixed $\lambda$ reduces to a standard Q learning task, jointly learning $\lambda^*$ introduces additional complexity due to the coupling between the dual variable and the value functions. Specifically, in the absence of system knowledge, the gradient required to update $\lambda$ cannot be computed directly--and, moreover, it cannot be estimated directly during the Q-learning process.

To overcome this challenge, we propose a three-timescale stochastic approximation algorithm. The algorithm updates the Q-function of $J_i(\lambda)$ and the dual variable $\lambda$ on medium and slow timescales, respectively. A third, faster timescale is employed to estimate the derivative of $g_i(\lambda)$ with respect to $\lambda$, which plays a critical role in guiding the update of $\lambda$.

\subsection{Useful Properties}
We will begin by introducing some useful definitions and establishing key properties of the optimization problem \eqref{eq: inf-lambda}, which form the basis of our method.
We first define an auxiliary MDP for each arm as follows:
\begin{definition}[Auxiliary MDP]
The auxiliary MDP associated with arm $\mathcal{B}_i$ is defined as $\mathcal{M}_i = (\mathcal{S}_i, \mathcal{A}_i, c, p_i)$, where the state space $\mathcal{S}_i$, action space $\mathcal{A}$, and transition kernel $p_i$ are identical to those of $\mathcal{B}_i$. The cost function is given by $c(s, a) = a$ for all $s \in \mathcal{S}_i$ and $a \in \mathcal{A}_i$.
\end{definition}

The long-term average cost of $\mathcal{M}_i$ under a policy $\pi$ is given by
\begin{align*}
	h_i^\pi(s) \triangleq \lim\limits_{T\to \infty}\frac{1}{T} \mathbb{E}_\pi \left[  \sum_{t=1}^{T} a_i^t |s^1_i = s \right], \ \forall s\in \mathcal{S}_i. 
\end{align*}
Since $\mathcal{B}_i$ is a unichain MDP, so is $\mathcal{M}_i$. As a result, the average cost $h_i^{\pi}(s)$ is independent of the initial state $s$, and we will henceforth denote it simply as $h_i^{\pi}$. Let $D^\pi_i(s,a)$ denote the state-action value function of $\mathcal{M}_i$ under policy $\pi$, and $\pi(a'|s')$ the probability of taking actin $a'$ in state $s'$. Then we have
\begin{align*}
	D_i^\pi(s,a) + h_i^\pi = a + \sum_{s'\in \mathcal{S}_i}  \sum_{a'\in \mathcal{A}} p_i(s'|s,a)\pi(a'|s') D_i^\pi(s',a').
\end{align*}
Clearly, $J_i(\lambda)$ and $\mathcal{M}_i$ share the same policy space. Let $\pi$ be a policy for the single-arm problem $J_i(\lambda)$, and $V_i^\pi(s)$ and $g_i^\pi$ be the associated value function and long-term average reward, respectively. 
It is easy to see that, for any policy $\pi$,
\begin{align}
    \frac{d g_i^\pi}{d \lambda} = -h_i^\pi, \ \forall i\in [M].
\end{align}
Let $\pi_i^\lambda$ denote the optimal policy for problem $J_i(\lambda)$, then the derivative of function $f(\lambda)$ is given by
\begin{align} \label{eq: f'}
	f'(\lambda) = \sum_{i=1}^{M} \frac{d g_i(\lambda)}{d \lambda} + N = N-\sum_{i=1}^{M} h^{\pi_i^\lambda}_i.
\end{align}
Drawing from these concepts, we can establish the following properties regarding $f(\lambda)$ and $\lambda^*$:
\begin{lemma} \label{lem: 1}
	For any RMAB with bounded reward functions $\{r_i\}_{i\in [M]}$, $f(\lambda)$ is a piecewise linear and convex function. In addition, there always exist a bounded $\lambda^*$ that achieves the minimum value of $f(\lambda)$.
\end{lemma}
\begin{IEEEproof}
	By definition, we have $g_i(\lambda) = \max_{\pi} g^\pi_i(\lambda)$.
	As mentioned above, for any policy $\pi$, the derivative ${d g_i^\pi}/{d \lambda} = -h_i^\pi\le 0$ is independent of $\lambda$, implying $g^\pi_i(\lambda)$ is a decreasing linear function of $\lambda$. Consequently, $g_i(\lambda)$ is decreasing, piecewise linear and convex given that it is the maximum of a set of decreasing linear functions. It then follows immediately that $f(\lambda)$ is piecewise linear and convex.
	
	Furthermore, since $a^t_i\in \{0,1\}$ for all $t$ and $i$, we have $0\le h^\pi_i \le 1$. Hence $-1 \le {d g_i}/{d \lambda} \le 0$ for all $i\in [M]$. It is easy to verify that, as $\lambda\to -\infty$, the optimal policy for $J_i(\lambda)$ is taking action 1 all the time. That is, 
	\begin{align*}
		\lim\limits_{\lambda\to -\infty} \frac{d g_i(\lambda)}{d \lambda} = -1, \ \forall i\in [M].
	\end{align*}
	Using the same argument yields $\lim_{\lambda\to \infty}{d g_i}/{d \lambda} = 0$. It follows that
	\begin{align*}
		\lim\limits_{\lambda\to -\infty}  f'(\lambda) = N-M <0, \ \lim\limits_{\lambda\to \infty}  f'(\lambda) = N >0.
	\end{align*}
	Therefore, there must exist a bounded $\lambda_0$ satisfying
	\begin{align*}
		f'(\lambda^-_0) \le 0, \  f'(\lambda^+_0) \ge 0.
	\end{align*}
	This $\lambda_0 = \lambda^*$ achieves the minimum of $f(\lambda)$.
\end{IEEEproof}

Lemma~\ref{lem: 1} guarantees the existence of a bounded optimal solution $\lambda^*$, thereby establishing the feasibility of learning $\lambda^*$ in practice.

\subsection{GINO-Q learning algorithm}

Our objective is to learn the optimal dual variable $\lambda^*$ and the corresponding Q-functions for the problems $J_i(\lambda^*)$, for all $i \in [M]$. To this end, we propose a simple yet effective algorithm that simultaneously updates $\lambda$ and learns the associated Q-functions, thereby enabling the computation of gain indices for all arms. We refer to this approach as Gain-Index-Oriented Q-learning (GINO-Q).

In particular, we fix a stepsize sequence $\{\theta^t:t\ge 1\}$, and employ the stochastic gradient-descent method to update $\lambda$:
\begin{align}  \label{eq: iter-lambda}
	\lambda^{t+1} = \lambda^t - \theta^t \hat{f}'(\lambda^t),
\end{align}
where $\hat{f}'(\lambda^t)$ is an estimator of $f'(\lambda^t)$, to be detailed later. Meanwhile, we use the standard relative value iteration (RVI) Q learning algorithm \cite{abounadi2001learning} to learn the Q function of every single-arm problem $J_i(\lambda^t)$ for the current $\lambda^t$. That is, $\forall i\in [M]$:
\begin{align}  \label{eq: iter-Q}
	Q_i^{t+1}(s^t_i,a^t_i) = & Q_i^{t}(s^t_i,a^t_i) + \beta^t_i [ r_i(s^t_i,a^t_i) - \lambda^t a^t_i + \notag \\
	&  \max_{a} Q_i^{t}(s^{t+1}_i,a) - Q_i^{t}(s^t_i,a^t_i) - g^t_i ],
\end{align}
where $\{\beta^t_i:t\ge 1\}$ is a predefined stepsize sequence. While $g_i^t$ can be estimated in various ways \cite{abounadi2001learning}, one of the most widely used methods is:
\begin{align} \label{eq: est_g}
	g_i^t = \frac{1}{2|\mathcal{S}|}\sum_{s\in \mathcal{S}_i}\sum_{a\in \mathcal{A}_i} Q^t_i(s,a) . 
\end{align}
We proceed by constructing the estimator $\hat{f}'(\lambda)$ according to \eqref{eq: f'}. Given any $\lambda$, the optimal policy $\pi^\lambda_i$ for $J_i(\lambda)$ selects actions greedily according to the optimal Q function of $J_i(\lambda)$, denoted by $Q_i$. At time step $t$, $Q^t_i$ serves as an estimate of $Q_i$. Hence, the policy that select actions greedily according to $Q^t_i$ can be regarded as an estimate of $\pi^\lambda_i$. To learn the average cost of this policy for the auxiliary MDP $\mathcal{M}_i$, we employ the SARSA algorithm \cite{sutton2018reinforcement}:
\begin{align}  \label{eq: iter-D}
	D_i^{t+1}(s^t_i,a^t_i) =   D_i^{t}(s^t_i,a^t_i) + \alpha^t_i [& a^t_i +   D_i^{t}(s^{t+1}_i,a_i^{t+1}) - \notag  \\
	&D_i^{t}(s^t_i,a^t_i) - h^t_i ],
\end{align}
where $\{\alpha^t_i:t\ge 1 \}$ is the stepsize sequence, action $a_i^{t+1}$ is selected greedily based on $Q^t_i$, that is, 
$$a_i^{t+1} = \arg\max_{a} Q_i^{t}(s^{t+1}_i,a).$$
As in \eqref{eq: iter-Q}, the average cost of $\mathcal{M}_i$ is estimated by
\begin{align} \label{eq: est_h}
	h_i^t = \frac{1}{2|\mathcal{S}|}\sum_{s\in \mathcal{S}_i}\sum_{a\in \mathcal{A}_i} D^t_i(s,a) .
\end{align}
We then estimate $f'(\lambda^t)$ as follows: 
$$\hat{f}'(\lambda^t) = N-\sum_{i=1}^{M} h^t_i.$$

The stepsize schedules of the three coupled iterates play a critical role in learning both $\lambda^*$ and the corresponding Q functions. As shown in equation~\eqref{eq: f'}, for any given $\lambda$, the derivative $f’(\lambda)$ depends on the optimal policy $\pi_i^\lambda$ for the single-arm problem $J_i(\lambda)$. Therefore, the update of the sequence $\{\lambda^t\}$ must proceed more slowly than that of $Q_i^t$, allowing Q-learning to sufficiently approximate $Q_i$. Conversely, since SARSA aims to estimate the state-action value function of $\mathcal{M}_i$ under the greedy policy induced by $Q^t_i$, its updates should operate at a faster timescale than Q-learning.

We define the stepsize sequences of the three coupled iterates properly so that they form a three-timescale stochastic approximation algorithm, with updates \eqref{eq: iter-lambda}, \eqref{eq: iter-Q}, and \eqref{eq: iter-D} operating in the slow, medium, and fast timescales, respectively. In particular, define\footnote{In practice, $t$ in the expressions of $\alpha^t_i $ and $\beta^t_i $ can be replaced by $n(s^t_i,a^t_i)$, the number of occurrences of $(s^t_i,a^t_i)$ up to time $t$. If $\lambda^t$ changes greatly compared with $\lambda^{t-1}$, we may reset $n(s,a)$ as if encountering a new RL problem. This may speed up convergence. }
\begin{align*}
	\alpha^t_i = \frac{C_1}{t}, \beta^t_i = \frac{C_2}{t \sqrt{\log t}}, \theta^t = \frac{C_3}{t \log t} \mathbf{1}\{t (\text{mod }C_4)=0\},
\end{align*}
where $\{C_k\}$ are constants, $\mathbf{1}\{\cdot\}$ is the indicator function, $\{\alpha^t_i \}$ and $\{\beta^t_i \}$ are invariant across $i\in[M]$.
It is easy to verify that $\sum_{t=1}^{\infty} x^t = \infty$ and $\sum_{t=1}^{\infty} (x^t)^2 < \infty$ for $x=\alpha_i, \beta_i$ or $\theta$. Update \eqref{eq: iter-D} is faster than \eqref{eq: iter-Q} because $\beta^t_i = o(\alpha_i^t)$. A similar argument indicates that update \eqref{eq: iter-Q} is faster than \eqref{eq: iter-lambda}.

The GINO-Q algorithm is summarized in Algorithm \ref{alg:ginoQ}. 
A detail to note is line 14, where $\lambda^t$ is updated only if the estimated derivative is decreased in absolute value. Since $f(\lambda)$ is piecewise linear and convex, this strategy helps mitigate erroneous updates caused by noise in the estimation of $f’(\lambda)$. While this may slightly slow convergence near $\lambda^*$, it improves stability and robustness during learning.

\begin{algorithm}[tb]
	\caption{GINO-Q Learning}
	\label{alg:ginoQ}
	\textbf{Input}: Integer $T$, learning rates $\{\alpha^t_i \}$, $\{\beta^t_i \}$, and $\{\theta^t \}$\\
	\textbf{Initialization}: Let $t=1, \lambda^1 = 0, y^0=M$. Reset the RMAB and receive the initial arm states $\{s^1_i \}$\\
	\textbf{Output}: Gain indices of all arms
	\begin{algorithmic}[1] 
		\WHILE{$t\le T$}
		\FOR{$i=1,2,\cdots,M$}
		\STATE Select action $a^t_i$ according to $Q^t_i$ (e.g., $\epsilon$-greedy)
		\STATE Apply action $a^t_i$ to the $i$-th arm, observe reward $r^t_i$ and the new state $s^{t+1}_i$
            \STATE Compute $g^t_i$ using eq. \eqref{eq: est_g}
		\STATE $\delta_i^{t} = r^t_i - \lambda^ta^t_i +  \max_a Q_i^{t}(s^{t+1}_i,a) - Q_i^{t}(s^t_i,a^t_i) - g^t_i$
		\STATE $Q_i^{t+1}(s^t_i,a^t_i) = Q_i^{t}(s^t_i,a^t_i) + \beta^t_i \delta_i^{t} $ 
            \STATE Compute $h^t_i$ using eq. \eqref{eq: est_h}
		\STATE $b^{t+1}_i  = \arg \max_a Q_i^{t}(s^{t+1}_i,a)$  
		\STATE $\sigma^t_i = a^t_i + D_i^{t}(s^{t+1}_i,b^{t+1}_i) - D_i^{t}(s^t_i,a^t_i) - h^t_i$
		\STATE $D_i^{t+1}(s^t_i,a^t_i) = D_i^{t}(s^t_i,a^t_i) +\alpha^t_i \sigma^t_i $ 
		\ENDFOR
		\STATE $y^t = N - \sum_{i=1}^{M} h_i^t $
		\STATE $\lambda^{t+1} = \lambda^t - \theta^t \mathbf{1}\{|y^t|<|y^{t-1}|\}y^t$
		\STATE $t\leftarrow t+1$
		\ENDWHILE
		\FOR{$i=1,2\cdots,M$}
		\FOR{$s\in \mathcal{S}_i$}
		\STATE $W_i(s) = Q_i^T(s,1) - Q_i^T(s,0) \ $  // gain index
		\ENDFOR
		\ENDFOR
	\end{algorithmic}
\end{algorithm}
\textit{Remark:}
If the learner has knowledge of the arm classification (but not the specific parameters of arms), updates \eqref{eq: iter-Q} and \eqref{eq: iter-D} do not need to be performed for each individual arm. Instead, it is sufficient to maintain a pair of $(Q^t_i, D^t_i)$ for each class. This is because arms within the same class share the same Q and D functions. As a result, the complexity of GINO-Q increases linearly with respect to the number of classes $K$. If the arm classification is unknown, then the complexity scales linearly with respect to the number of arms $M$, which still represents a significant reduction compared to the exponential growth of the state space.

\section{Discussion} \label{sec: discussion}
The definition of gain indices is dependent on $\lambda^*$, and our GINO-Q learning aims to learn $\lambda^*$ and the corresponding Q function of every single-arm problem $J_i(\lambda^*)$. This section discusses the robustness of GINO-Q learning with respect to the value of $\lambda^*$. Remarkably, we show that the asymptotic optimality of the gain index policy is assured as long as the sequence $\{\lambda^t \}$ converges to a neighbourhood of $\lambda^*$---not necessarily converging precisely to $\lambda^*$. This implies that, as least when $M$ is large, convergence to a neighbourhood of $\lambda^*$ is sufficient to induce a near-optimal index policy.

Formally, for any $\lambda$ and any arm $\mathcal{B}_i$, define
\begin{align*}
	W^\lambda_i(s) \triangleq Q_i(s,1,\lambda) - Q_i(s,0,\lambda), \ \forall s\in \mathcal{S}_i. 
\end{align*}
Then the gain index is $W_i(s) = W^{\lambda^*}_i(s)$. Learning the values of $\{W^\lambda_i(s) \}$ for a given $\lambda$ constitutes a well-studied reinforcement learning task. However, a practical concern arises: how does the performance of the index policy get affected by inaccuracies in estimating $\lambda^*$? The following lemma partially addresses this question by showing that the policy remains asymptotically optimal to a certain degree of estimation error in $\lambda^*$.

\begin{lemma} \label{lem: 2}
	There exists a non-empty interval $(\lambda^l, \lambda^u)$ that contains  $\lambda^*$, such that the gain index policy with indices $\{W^\lambda_i(s) \}$ is asymptotically optimal for any $\lambda\in(\lambda^l, \lambda^u)$.
\end{lemma}
\begin{IEEEproof}
	Theorem 1 establishes the asymptotic optimality by demonstrating that, as $M\to \infty$, the gain index policy achieves an upper bound on the optimal value of the RMAB. This upper bound corresponds to the optimal value of the relaxed RMAB problem, which is defined by objective (1) subject to constraint (3). We will refer to the optimal policy for the relaxed RMAB problem as the OR policy. Let $x^t$ represent the number of arms activated by the OR policy at time $t$. The proof relies on two key facts: (i) when $x^t=N$, both the OR policy and the gain index policy activate the same set of arms at time $t$; (ii) as $M\to \infty$, $x^t/M$ converges to $N/M$. These facts suggest that in the asymptotic scenario, the gain index policy performs nearly identically to the OR policy. Interested readers can refer to \cite{TIT2023} for a detailed proof. 
	
	Let $\mu_\lambda$ denote the policy for the relaxed RMAB that activates the $i$-th arm according to policy $\pi^\lambda_i$. Then the OR policy corresponds to policy $\mu_{\lambda^*}$. In this proof, we show the existence of a non-empty interval $(\lambda^l,\lambda^u)$ such that for any  $\lambda$ within this interval: (1) $\mu_\lambda$ is identical to $\mu_{\lambda^*}$ ; and (2) policy $\mu_\lambda$ and the gain index policy defined by $\{W^\lambda_i(s)\}$ share the relationship described by facts (i) and (ii). This means that in the asymptotic scenario, the gain index policy defined by $\{W^\lambda_i(s)\}$ performs nearly identically to the OR policy, thereby proving the asymptotic optimality.   
	
	Consider the following definition 
	\begin{align*}
		\mathcal{E}^\lambda_i \triangleq \{s\in \mathcal{S}_i: W^\lambda_i(s)>0 \},\ i \in [M].
	\end{align*}
	For the $i$-th arm, $\mathcal{E}^\lambda_i $ is the set of states in which policy $\pi^\lambda_i$ takes action 1. Now, for any $\lambda$, if policy $\mu_\lambda$ activates exactly $N$ arms at some time $t$, then these $N$ activated arms have positive indices $W^\lambda_i(s)$, while the remaining arms have non-positive indices at that time. As a result, the index policy with indices $\{W^\lambda_i(s) \}$ will activate the $N$ arms with positive indices, as they occupy the top $N$ positions. In summary, when policy $\mu_\lambda$ activates $N$ arms, the index policy with indices $\{W^\lambda_i(s) \}$ will match the arm selection of policy $\mu_\lambda$.
	
	As stated in Lemma 1, $f(\lambda)$ is a piecewise linear and convex function, and there exists a bounded $\lambda^*$ that achieves the minimum value of $f(\lambda)$. Then it is easy to see that $\lambda^*$ satisfies one of the following conditions: (1) $f'(\lambda^*)=0$; (2) $f'(\lambda^*)$ is not defined and 
	\begin{align*}
		\lim\limits_{\lambda \uparrow \lambda^*} f'(\lambda) <0, \ \lim\limits_{\lambda \downarrow \lambda^*} f'(\lambda) >0.
	\end{align*}
	
	We start from case (2), $\lambda^*$ is the common endpoint of two segments, as shown in fig. 1 of the paper. Let $\lambda^l$ denote the left endpoint of the left segment, $\lambda^u$ denote the right endpoint of the right segment. Since $g_i(\lambda)$ is convex for all $i$, $f'(\lambda)=N-\sum_{i=1}^{M}g'_i(\lambda)$ remains constant within interval $(\lambda^l, \lambda^*)$ implies that $\pi^\lambda_i$ remains unchanged within this interval for all $i$. Using the same argument yields that $\pi^\lambda_i$ remains unchanged within interval $(\lambda^*, \lambda^u)$  for all $i$. Let $\mu^-$ and $\mu^+$ denote policies $\mu_\lambda$ for $\lambda\in(\lambda^l, \lambda^*)$ and $\lambda\in(\lambda^*, \lambda^u)$, respectively. Then $\mu^-$ and $\mu^+$ are equally optimal for the relaxed RMAB. Hence $x^t/M$ converges to $N/M$ as $M\to \infty$ for both policies.
	
	For the case of $f'(\lambda^*)=0$,  there must exists an interval $(\lambda^l, \lambda^u)$ such that $f'(\lambda)=0$ for $\lambda\in(\lambda^l, \lambda^u)$. According to the same argument as before, $\pi^\lambda_i$ remains unchanged within the interval $(\lambda^l, \lambda^u)$ and it is the OR policy.
	
	In summary, there exists a non-empty interval $(\lambda^l,\lambda^u)$ containing $\lambda^*$ such that, for any $\lambda$ within this interval, the index policy with indices $\{W^\lambda_i(s) \}$ performs nearly identically to the OR policy in the asymptotic scenario. Hence it is asymptotic optimal. The remaining proof follows standard techniques; for further details, refer to \cite{TIT2023}.
\end{IEEEproof}

\begin{figure}[t]
	\centering
	\includegraphics[width=2.3in]{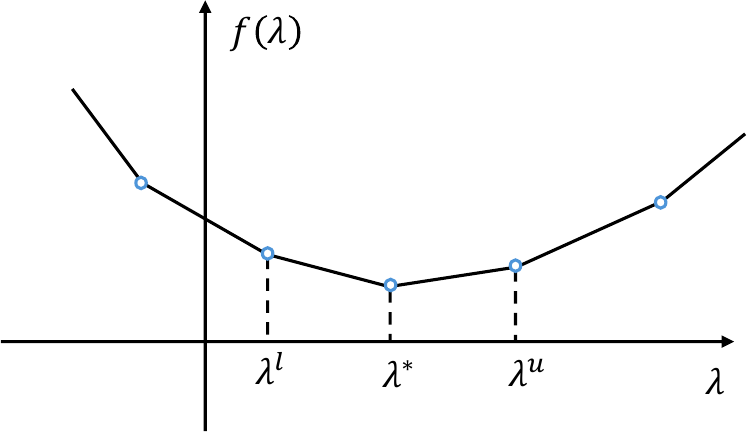}
	\caption{An illustration of  $f(\lambda)$ and $(\lambda^l,\lambda^u)$.}
	\label{fig:func}	
\end{figure}

As demonstrated in the proof, $(\lambda^l,\lambda^u)$ is the interval where $f'(\lambda)=0$, whenever such an interval exists. If no such interval exists, $(\lambda^l,\lambda^u)$ corresponds to the range of the two segments connected to $\lambda^*$ (as shown in Figure \ref{fig:func}). In this latter case, the non-smoothness of $f(\lambda)$ may pose a challenge when searching for $\lambda^*$ using gradient decent methods. Specifically, the sequence $\{\lambda^t \}$ may oscillate around $\lambda^*$ instead of converging to it precisely because $|f'(\lambda)|$ is bounded away from 0 for $\lambda$ in the neighborhood of $\lambda^*$. Fortunately, according to Lemma \ref{lem: 2}, this situation will not affect the asymptotic optimality of the resulting index policy. This property makes GINO-Q particularly well-suited for large-scale RMABs---when $M$ is large (with relatively small number of classes), GINO-Q can learn a near-optimal policy and remains robust to the inaccuracies in $\lambda^*$.

\begin{figure*}[t] \label{fig:non-idex}
	\centering
	\begin{subfigure}[b]{0.32\linewidth} 
		\includegraphics[width=\linewidth]{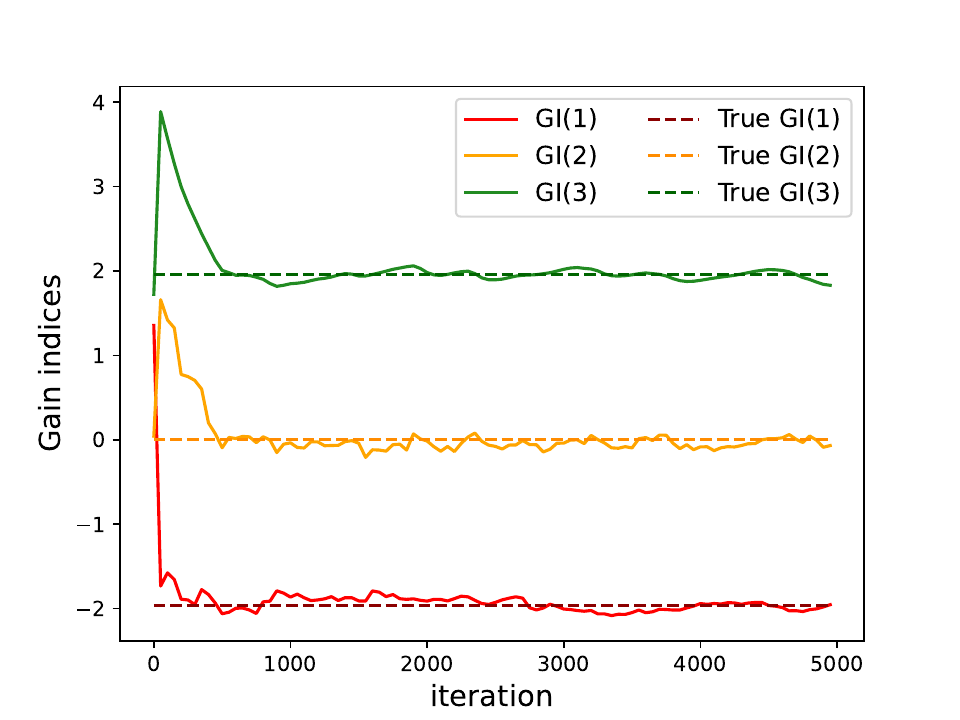}
		\caption{Gain index over training.}
		\label{fig:idx}
	\end{subfigure}
	\begin{subfigure}[b]{0.32\linewidth}
		\includegraphics[width=\linewidth]{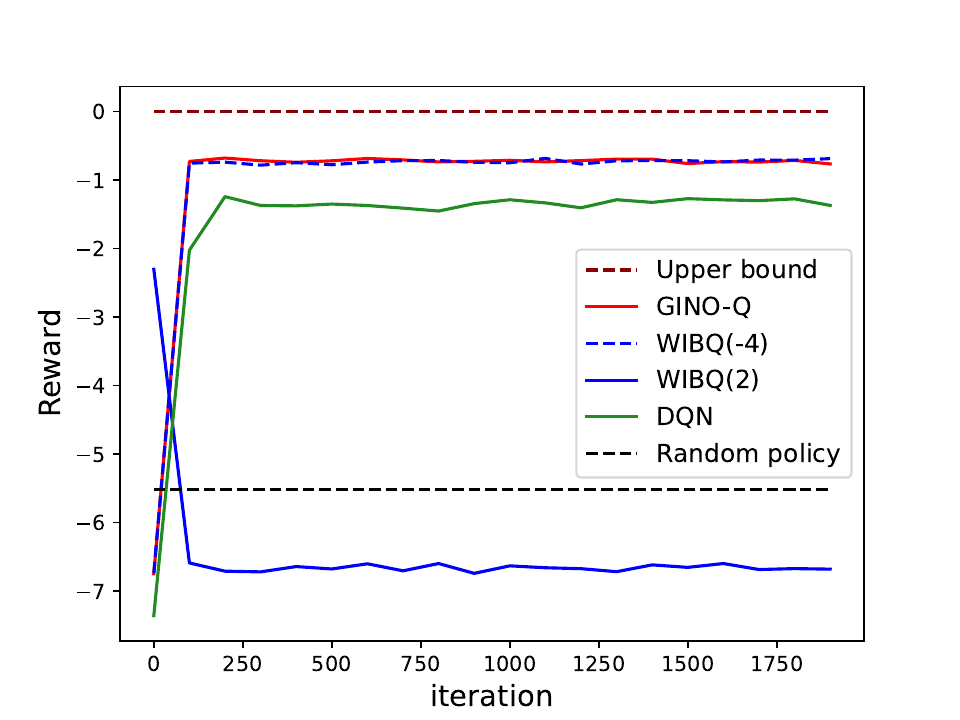}
		\caption{Reward over training ($M=10,N=7$).}
		\label{fig:nonID10}
	\end{subfigure}
	\begin{subfigure}[b]{0.32\linewidth}
		\includegraphics[width=\linewidth]{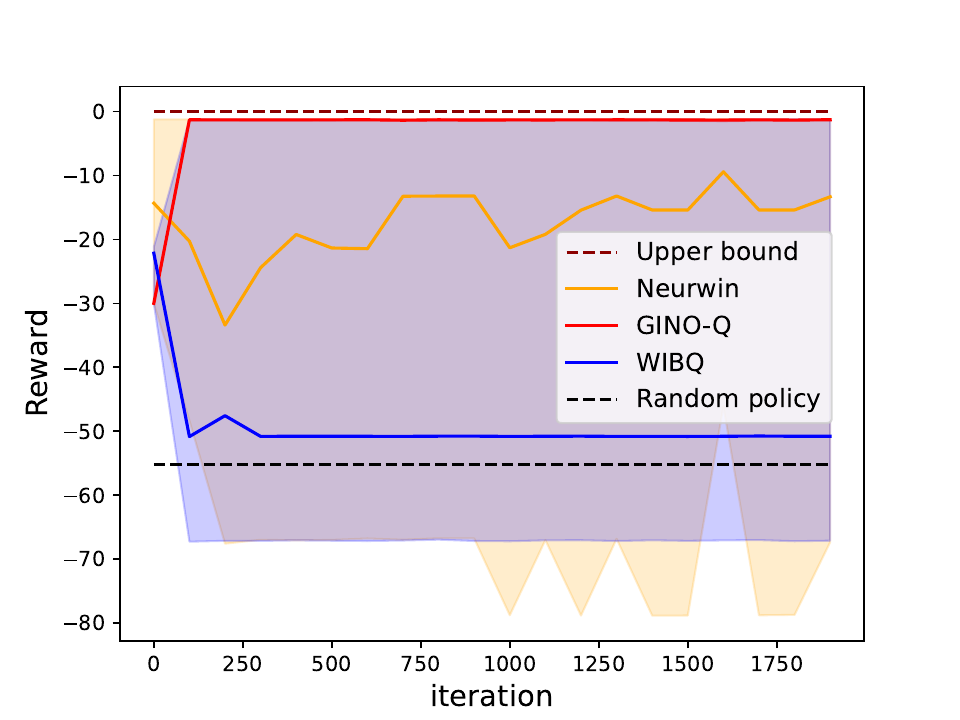}
		\caption{$M=100,N=70$.}
		\label{fig:nonID100}
	\end{subfigure}
	\caption{Performances of the GINO-Q and baseline algorithms in a non-indexable RMAB problem. In (a), GI($s$) denotes the gain index of state $s$. States $3$ to $6$ share the same gain index, hence we only plot GI(3) and omit the others. In (b), WIBQ($x$) corresponds to the result of a single run where the Whittle index of state 1 learned by WIBQ is $x\in \{-4, 2\}$. In (c), the results represent the average of 20 independent runs, with the shaded areas indicating confidence bounds. }
\end{figure*}	
\begin{figure*}[t] \label{fig:AoI}
	\centering
	\begin{subfigure}[b]{0.32\linewidth} 
		\includegraphics[width=\linewidth]{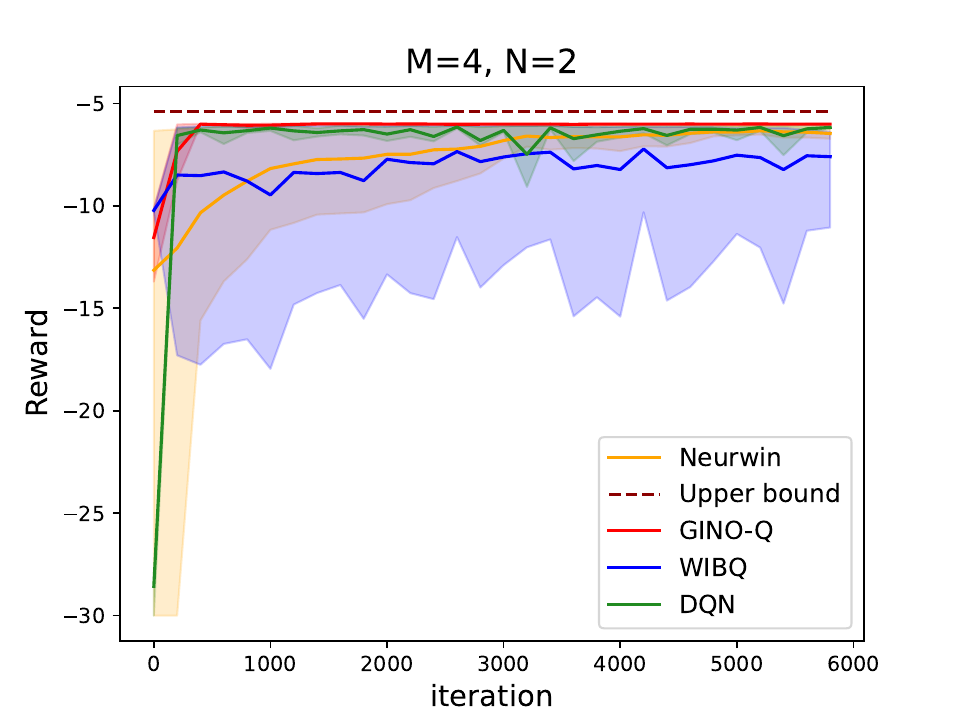}
	\end{subfigure}
	\begin{subfigure}[b]{0.32\linewidth}
		\includegraphics[width=\linewidth]{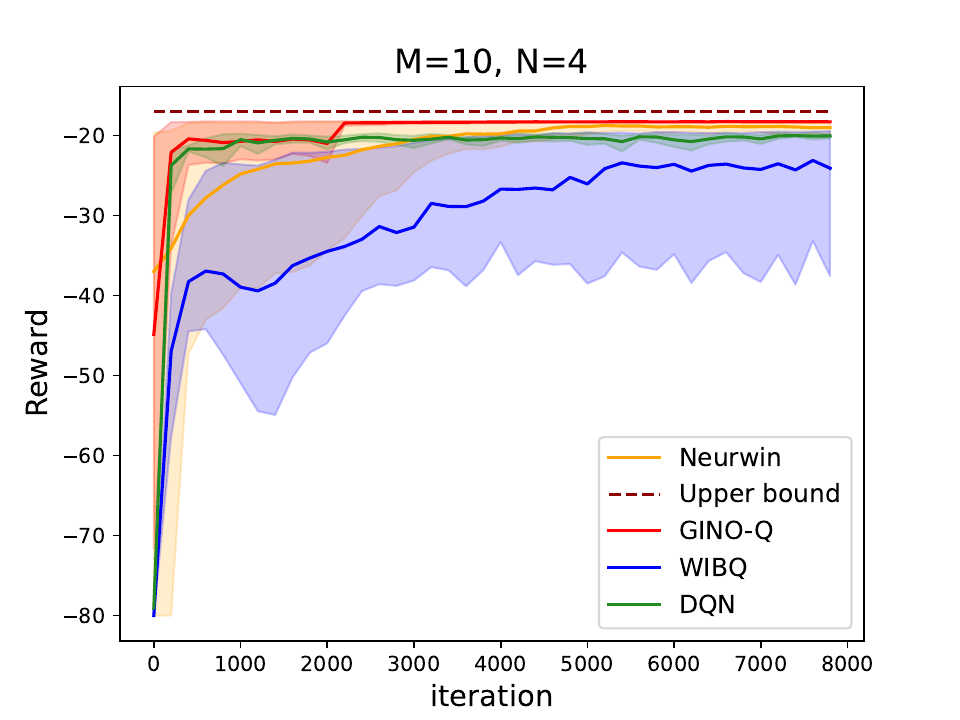}
	\end{subfigure}
	\begin{subfigure}[b]{0.32\linewidth}
		\includegraphics[width=\linewidth]{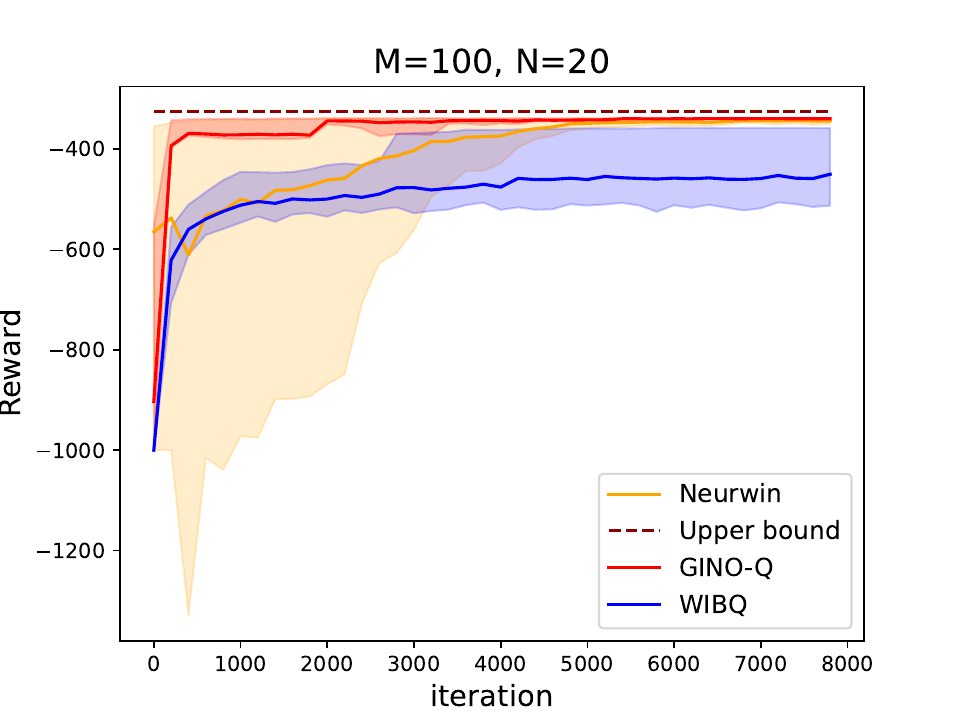}
	\end{subfigure}
	\caption{Performance comparisons between GINO-Q and baseline algorithms in the channel allocation problem.}
\end{figure*}

\section{Experiments}  \label{sec: exps}
In this section, we showcase the performance of GINO-Q learning by evaluating it across three distinct RMAB problems. For each problem, we explore various settings characterized by different pairs of $(M,N)$ values, where $M$ is the number of arms, and $N$ denotes the number of arms to be activated at each time. 

The baseline algorithms include the conventional reinforcement learning method DQN \cite{mnih2015human}, as well as recently developed Whittle-index-based approaches: WIBQ \cite{avrachenkov2022whittle} and Neurwin \cite{nakhleh2021neurwin}. Among these, DQN models the entire RMAB as a single MDP. Under this formulation,  the state space grows exponentially in $M$, and the number of valid actions is $M\choose N$. As $M$ increases, the RMAB rapidly becomes an MDP with large discrete state and action spaces, a scenario that DQN struggles to handle. Therefore, we only evaluate DQN in small-scale problems. On the other hand, WIBQ and Neurwin are state-of-the-art algorithms specifically designed for RMABs, with a focus on learning the Whittle index policy. They serve as our primary baselines for comparison. Additional details of the experiments can be found in the appendix.

\subsection{Non-indexable RMAB} \label{subsec: non-index}
Let us first compare the proposed GINO-Q learning with the Whittle-index-based learning algorithms in a non-indexable RMAB. Both WIBQ and Neurwin assume that the given RMAB is indexable because the Whittle index policy is only defined for indexable RMABs. However, determining the indexability of an RMAB can be challenging, particularly in the absence of system knowledge. What happens if we apply a Whittle-index-based learning algorithm to a non-indexable problem? To answer this question, we constructed a non-indexable RMAB with arms as defined in Example 1 and assessed the performances of  GINO-Q and  WIBQ  in this setting. 

We first consider the setting with $M=10$ and $N=7$. Figure \ref{fig:idx} demonstrates that GINO-Q effectively acquires the gain indices. It is important to note that the key point of the gain index policy is the relative ordering of states by their gain indices. Consequently, the performance of GINO-Q stabilizes once this order is established (compare Figure \ref{fig:idx} with \ref{fig:nonID10}). As discussed in {Example 1}, there are two values of $\lambda$ (i.e., $-4$ and $2$) that satisfy the condition $Q_i(1,1,\lambda)=Q_i(1,0,\lambda)$. As a result, both can be recognized by the WIBQ algorithm as valid Whittle indices for state 1. In contrast, each of the remaining states admits a unique Whittle index.
In our experiments, WIBQ consistently learned the true Whittle indices for states 2 to 6. However, for state 1, the Whittle index converged to $-4$ in some runs and to $2$ in others. State 1 has the highest priority with index 2 and the lowest priority with index $-4$, resulting in significantly different performances, as shown in Figure \ref{fig:nonID10}.  
For benchmarking purposes, we also evaluated the performance of DQN and the random policy (i.e., selecting $N$ arms randomly at each step). Notably, WIBQ performs even worse than the random policy when the Whittle index of state 1 converges to $2$, highlighting the risk of applying Whittle-index-based learning to non-indexable RMABs. For simplicity, the results of Neurwin are not plotted in this setting as WIBQ already learned the true Whittle indices. 

We also compared the algorithms in the setting of $(M,N)=(100,70)$, as depicted in Figure \ref{fig:nonID100}. The low average of WIBQ suggests that it is more likely to converge to $2$ than $-4$. While Neurwin generally outperforms WIBQ in this RMAB, it similarly faces a high risk of learning a very poor index policy. In contrast, GINO-Q consistently learns a near-optimal index policy.
Moreover, we computed upper bounds for both settings, provided by the optimal value of the relaxed RMAB. 
Figure \ref{fig:nonID100} shows that when $M$ is large, GINO-Q closely approaches the upper bound, providing strong evidence of its asymptotic optimality.

\subsection{Channel Allocation in Communication Networks}
Channel allocation in communication networks involves strategically assigning communication channels to users to optimize the overall system performance. In a simple scenario, a network possesses a finite number of channels, each of which can be allocated to a single user at any time step.
A recent research topic is channel allocation aimed at minimizing the average age of information (AoI). AoI is a metric that measures information freshness, defined as the time elapsed since a user last received a message. \cite{tripathi2019whittle} formulated this problem as an RMAB and proved that this RMAB is indexable.


In this experiment, we consider the min-AoI channel allocation problem with two user classes. Each class has a different level of urgency for receiving fresh information, leading us to define a unique cost function of AoI for the class.
Let $x$ denote the AoI of a user. The cost functions for the two classes are defined as $f_1(x)=2x$ and $f_2(x)=\log x$, respectively. The objective is to minimize the overall cost across the network.

We evaluated the GINO-Q and baseline algorithms across three different scales of $M$ and $N$. The reported results represent averages from 20 independent runs, as presented in Figure 4. It can be observed that GINO-Q consistently achieves the best performance across all settings. While Neurwin is capable of learning a Whittle index policy comparable to the gain index policy, it converges at a significantly slower rate. Additionally, compared to Neurwin and WIBQ, GINO-Q exhibits very low variance, as indicated by their confidence bounds.

\begin{figure*}[t] \label{fig:pat}
	\centering
	\begin{subfigure}[b]{0.32\linewidth} 
		\includegraphics[width=\linewidth]{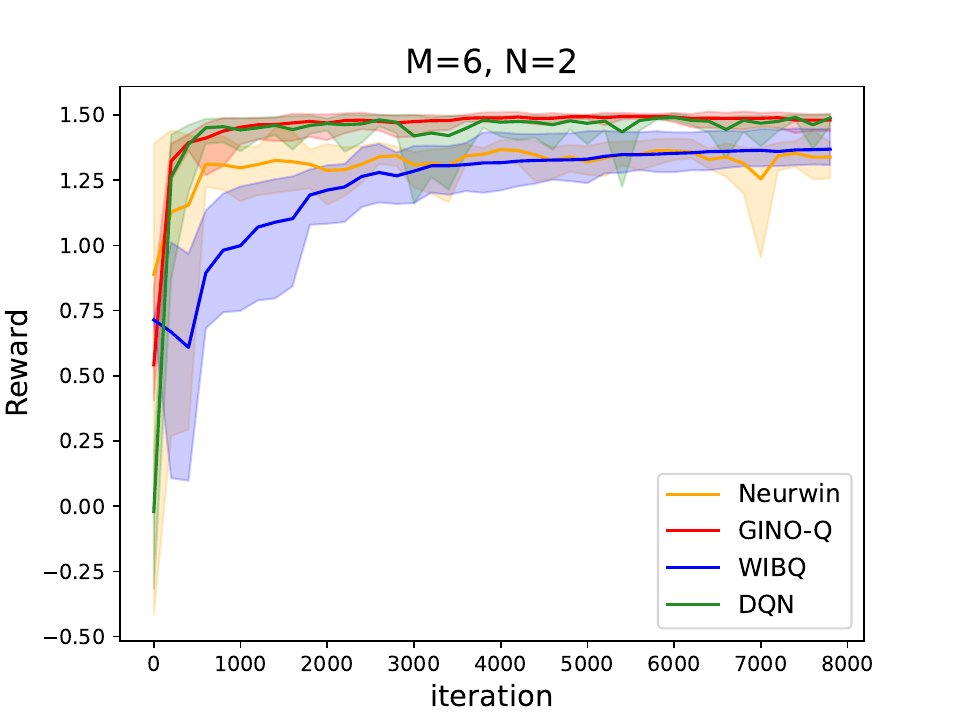}
	\end{subfigure}
	\begin{subfigure}[b]{0.32\linewidth}
		\includegraphics[width=\linewidth]{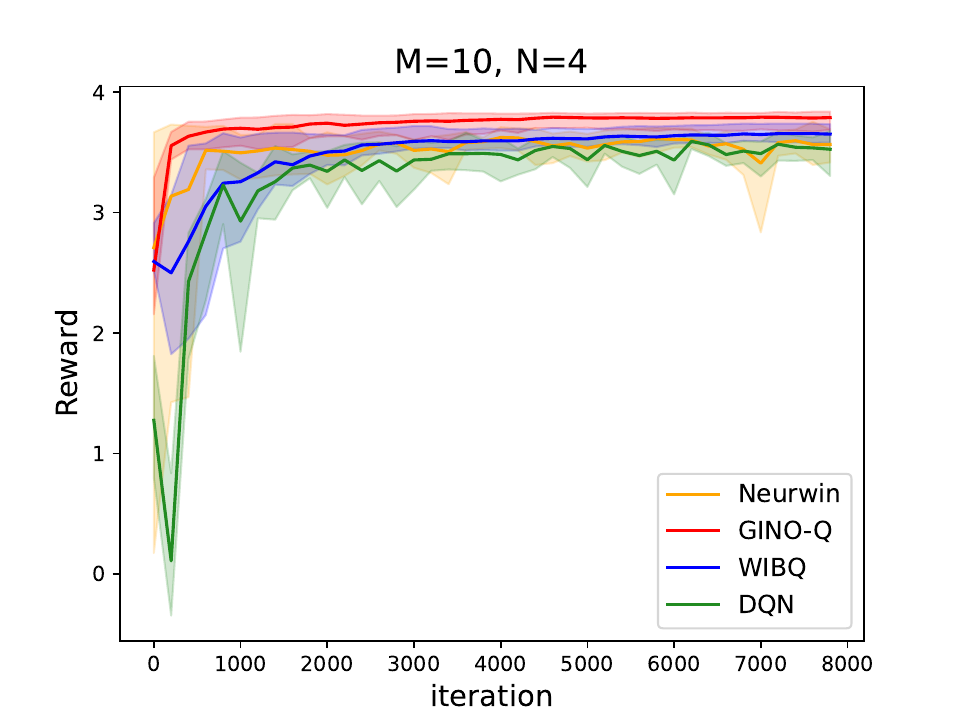}
	\end{subfigure}
	\begin{subfigure}[b]{0.32\linewidth}
		\includegraphics[width=\linewidth]{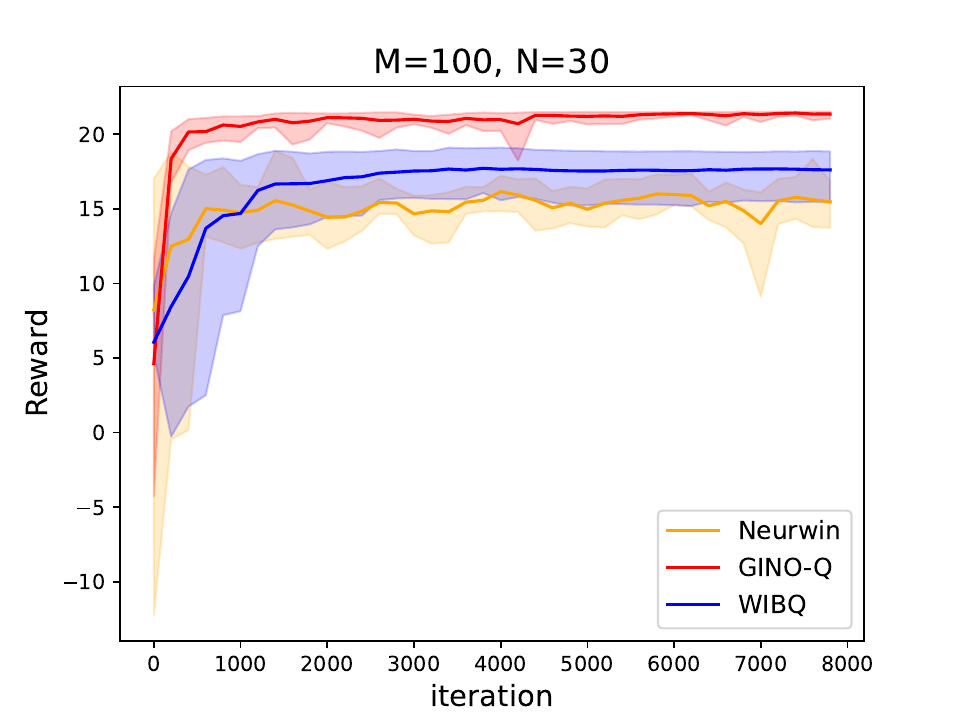}
	\end{subfigure}
	\caption{Performance comparisons between GINO-Q and baseline algorithms in the patrol scheduling problem.}
\end{figure*}
\subsection{Patrol Scheduling}
Patrol scheduling is another application that is often formulated as an RMAB \cite{xu2021dual}. Consider $N$ agents responsible for patrolling $M$ sites. Each agent can visit one site at a time to check for an event of interest. The event is represented as a binary Markov process, where state 1 indicates the event occurs, and state 0 indicates otherwise. The event of interest in a site could be, for example, whether there are illegal intruders in the area. When an agent patrols a site where the event occurs, the controller receives a reward. Conversely, if an event occurs at a site but no agent patrols it, a punishment is generated. In other cases, the controller receives neither reward nor punishment.

Clearly, each site corresponds to an arm in the RMAB. We define two classes of sites, each with a distinct transition matrix for the Markov event. The performances of GINO-Q and the baseline algorithms in this problem are reported in Fig. 5. Unfortunately, the upper bounds obtained from the relaxed RMAB appear to be loose for these settings, so they are not plotted here.
Nevertheless, once again, GINO-Q outperforms all benchmark algorithms across all settings. In the small-scale setting ($M=6$), DQN performs as well as GINO-Q. However, as $M$ increases to 10, DQN’s policy becomes less effective than both the gain and Whittle index policies due to the curse of dimensionality. When $M=100$, the RMAB exhibits enormous state and action spaces that DQN cannot handle. In contrast, GINO-Q is able to learn a near-optimal index policy quickly, even for large-scale RMABs. Its convergence speed is not affected by the increase in $M$, as long as the number of classes is fixed.

\section{Conclusion}  \label{sec: conclusion}
In this paper, we introduced GINO-Q, a novel three-timescale stochastic approximation algorithm designed to address the challenges posed by RMABs. Our approach effectively tackles the curse of dimensionality by decomposing the RMAB into manageable single-arm problems, ensuring that the computational complexity grows linearly with the number of arms. Unlike existing Whittle-index-based learning algorithms, GINO-Q does not require RMABs to be indexable, significantly broadening its applicability.
We showed experimentally that Whittle-index-based learning algorithms can perform poorly in non-indexable RMABs. In contrast, GINO-Q consistently learns near-optimal policies across all experimental RMABs, including non-indexable ones, and shows great efficiency by converging significantly faster than existing baselines.

\appendix
This appendix provides details of experimental settings and the parameters used in the algorithms.
\subsection{Non-indexable RMAB}
The construction of this non-indexable RMAB is inspired by the discussion in \cite{RMAB_Whittle1988}.
The parameters of this RMAB are thoroughly discussed in the main body of our paper and are not repeated here. In this part, we provide additional details on the convergence of the Whittle index of state 1 under the WIBQ algorithm. The value to which the Whittle index of state 1 converges depends on both the random seed and the initial index value. For example, if the seed is $5$ and the initial index value is $0$, then the Whittle index of state 1 converges to $-4$; while if the seed is $387$ and the initial index value is $-6$, then the Whittle index converges to $2$. Notably, even when the initial index value is closer to $-4$, it may still converge to $2$. In addition, using the Neurwin algorithm, we also observed that the Whittle index of state 1 oscillates between $-4$ and $2$.

The parameters of GINO-Q are as follows: $C_1=C_2=1, C_3=20, C_4=200$.

\subsection{Channel Allocation in Communication Networks}
In the channel allocation problem, we consider a network with $M$ users and $N$ channels. We assume that the $N$ channels are unreliable but identical for all users. Specifically, each time a user transmits over a channel, there is a probability of $1-\rho$ that the transmission fails and a probability of $\rho$ that it succeeds. If the transmission is successful, the user's AoI resets to 1; otherwise, the user's AoI increases by 1. 

Let $s^t_i$ denote the AoI of the $i$-th user at time $t$. Theoretically, $s^t_i$ can take any integer value, meaning the state space of each arm is the set of all positive integers, which is infinite. To simplify the problem, we truncate the state space to $\mathcal{S}_i=\{1, 2, \ldots, 100\}$ for all $i$, assuming a maximum AoI of 100 for each user. This truncation is a reasonable approximation in practical applications. For instance, in many cases, if a user does not receive new information within a certain period, they may be considered disconnected and cease functioning. Consequently, the cost associated with AoI for such a user reaches an upper bound and does not increase further with AoI.

Given this setting, the transition probabilities of AoI in our experiments are formally defined as follows: (i) if $a^t_i=0$, then $s^{t+1}_i = \min \{s^t_i+1, 100\}$ with probability 1; (ii) if $a^t_i=1$, then $s^{t+1}_i = \min \{s^t_i+1, 100\}$ with probability $1-\rho$, and $s^{t+1}_i = 1$ with probability $\rho$.
In all three scenarios, we set $\rho = 0.7$. Note that all users share the same transition kernel. However, users in different classes have different cost functions: (i) For users in class 1, the reward (i.e., negative cost) function is $r(s,a) = -2s$ for $a \in \{0,1\}$. (ii) For users in class 2, the reward function is $r(s,a) = -\log(s)$ for $a \in \{0,1\}$.

The parameters of GINO-Q used in this experiment are as follows: $C_1=C_2=1, C_3=20, C_4=200$.

\subsection{Patrol Scheduling}
In the patrol scheduling problem, let $x^t_i\in \{0,1\}$ denote the random variable of the Markov event in site $i$ at time $t$, where $x^t_i=1$ indicates the event happens, and $x^t_i=0$ otherwise. For each $i\in [M]$, $\{x^t_i:t\ge 1\}$ is a Markov chain that is independent of the patrol scheduling. We set two classes of sites in the experiments: (1) For sites in class 1, the transition kernel of the Markov event is given by
\begin{align*}
	&P(x^{t+1}_i=1|x^t_i=0)\triangleq q_{01} = 0.2, \\ 
       &P(x^{t+1}_i=0|x^t_i=1)\triangleq q_{10} = 0.1.
\end{align*}
(2) For sites in class 2, the transition kernel of the Markov event is given by
\begin{align*}
	&P(x^{t+1}_i=1|x^t_i=0)\triangleq \rho_{01} = 0.3, \\ 
        &P(x^{t+1}_i=0|x^t_i=1)\triangleq \rho_{10} = 0.4.
\end{align*}

The state of arm $i$ is represented by the probability that $x^t_i=1$. That is,
\begin{align*}
	s^t_i = P(x^t_i=1|x^{t'}_i)
\end{align*}
where $t'< t$ the latest time point that the $i$ site is patrolled. Given the transition kernel of the Markov event, $s^t_i$ can be determined by $x^{t'}_i$ and $t-t'$. In particular, let $n=t-t'$, then for arms that belong to class 1,
\begin{align*}
	s^t_i = \begin{cases}
		y_0(n)\triangleq\frac{{{q_{01}} - {q_{01}}{{(1 - {q_{01}} - q_{10})}^n}}}{{q_{01} + q_{10}}}, & \text{ if } x^{t'}_i=0 \\
		y_1(n)\triangleq 1 - \frac{{q_{10} -q_{10}{{(1 - q_{01}- q_{10})}^n}}}{q_{01} + q_{10}} & \text{ if } x^{t'}_i=1
	\end{cases}
\end{align*}
Note that both $y_0(n)$ and $y_1(n)$ converge to the same point:
\begin{align*}
	\lim_{n\to \infty} y_0(n) = \lim_{n\to \infty}y_1(n) = \frac{q_{01}}{q_{01} + q_{10}}\triangleq \omega
\end{align*}
Arms in class 2 share the same expression, just replacing $q_{01}$ and $q_{10}$ by $\rho_{01}$ and $\rho_{10}$, respectively.

We observed that $y_0(n) = y_1(n)\approx\omega$ for $n\ge 60$. Therefore, in our experiments, we constructed a finite-state approximation of the single-arm problem by truncating the original state space, which is an infinite set, to a finite-state space $\mathcal{S}_i = \{\omega, y_0(n), y_1(n):1\le n\le 60 \}$. Let $\tau(n) \triangleq \min \{60, n+1\}$, $s^t_i = y_k(n)$ for $k\in \{0,1\}$, then the transition kernel of this finite-state approximation is
\begin{align*}
	p_i(s^{t+1}_i|s^t_i,a^t_i) = \begin{cases}
		1, & \text{if }s^{t+1}_i=y_k(\tau(n)), a^t_i = 0 \\
		y_k(n), & \text{if }s^{t+1}_i=y_1(1), a^t_i = 1 \\
		1-y_k(n), & \text{if }s^{t+1}_i=y_0(1), a^t_i = 1 \\
	\end{cases}
\end{align*}
The reward for the patrol scheduling is set as 
\begin{align*}
	r(x^t_i,a^t_i) = \begin{cases}
		2, & \text{if }x^t_i=1, a^t_i=1 \\
		-1, & \text{if }x^t_i=1, a^t_i=0 \\
		0, & \text{otherwise}
	\end{cases}
\end{align*}
For the single-arm problem, we adopt the average reward for each state $s^t_i$. That is, 
\begin{align*}
	r(s^t_i,a^t_i) = 2s^t_i \mathbf{1}\{a^t_i=1\} - s^t_i \mathbf{1}\{a^t_i=0\} 
\end{align*}
where $\mathbf{1}\{\cdot\}$ is the indicator function.

The parameters of GINO-Q used in this experiment are as follows: $C_1=C_2=1, C_3=2, C_4=200$.

\bibliographystyle{IEEEtran}
\bibliography{reference}

\begin{thebibliography}{10}
\providecommand{\url}[1]{#1}
\csname url@samestyle\endcsname
\providecommand{\newblock}{\relax}
\providecommand{\bibinfo}[2]{#2}
\providecommand{\BIBentrySTDinterwordspacing}{\spaceskip=0pt\relax}
\providecommand{\BIBentryALTinterwordstretchfactor}{4}
\providecommand{\BIBentryALTinterwordspacing}{\spaceskip=\fontdimen2\font plus
\BIBentryALTinterwordstretchfactor\fontdimen3\font minus
  \fontdimen4\font\relax}
\providecommand{\BIBforeignlanguage}[2]{{%
\expandafter\ifx\csname l@#1\endcsname\relax
\typeout{** WARNING: IEEEtran.bst: No hyphenation pattern has been}%
\typeout{** loaded for the language `#1'. Using the pattern for}%
\typeout{** the default language instead.}%
\else
\language=\csname l@#1\endcsname
\fi
#2}}
\providecommand{\BIBdecl}{\relax}
\BIBdecl

\bibitem{adelman2008WCDP}
D.~Adelman and A.~J. Mersereau, ``Relaxations of weakly coupled stochastic
  dynamic programs,'' \emph{Operations Research}, vol.~56, no.~3, pp. 712--727,
  2008.

\bibitem{wang2019whittle}
J.~Wang, X.~Ren, Y.~Mo, and L.~Shi, ``Whittle index policy for dynamic
  multichannel allocation in remote state estimation,'' \emph{IEEE Transactions
  on Automatic Control}, vol.~65, no.~2, pp. 591--603, 2019.

\bibitem{wang2021restless}
K.~Wang and L.~Chen, \emph{Restless Multi-Armed Bandit in Opportunistic
  Scheduling}.\hskip 1em plus 0.5em minus 0.4em\relax Springer, 2021.

\bibitem{mate2022field}
A.~Mate, L.~Madaan, A.~Taneja, N.~Madhiwalla, S.~Verma, G.~Singh, A.~Hegde,
  P.~Varakantham, and M.~Tambe, ``Field study in deploying restless multi-armed
  bandits: Assisting non-profits in improving maternal and child health,'' in
  \emph{Proceedings of the AAAI Conference on Artificial Intelligence},
  vol.~36, no.~11, 2022, pp. 12\,017--12\,025.

\bibitem{RMAB_PSPACEhard}
C.~H. Papadimitriou and J.~N. Tsitsiklis, ``The complexity of optimal queuing
  network control,'' \emph{Mathematics of Operations Research}, vol.~24, no.~2,
  pp. 293--305, 1999.

\bibitem{RMAB_Whittle1988}
P.~Whittle, ``Restless bandits: Activity allocation in a changing world,''
  \emph{Journal of Applied Probability}, vol.~25, pp. 287--298, 1988.

\bibitem{Fu2019Q4WI}
J.~Fu, Y.~Nazarathy, S.~Moka, and P.~G. Taylor, ``Towards q-learning the
  whittle index for restless bandits,'' in \emph{2019 Australian \& New Zealand
  Control Conference (ANZCC)}, 2019, pp. 249--254.

\bibitem{avrachenkov2022whittle}
K.~E. Avrachenkov and V.~S. Borkar, ``Whittle index based q-learning for
  restless bandits with average reward,'' \emph{Automatica}, vol. 139, p.
  110186, 2022.

\bibitem{xiong2023finite}
G.~Xiong and J.~Li, ``Finite-time analysis of whittle index based q-learning
  for restless multi-armed bandits with neural network function
  approximation,'' in \emph{Advances in Neural Information Processing Systems},
  vol.~36, 2023, pp. 29\,048--29\,073.

\bibitem{nakhleh2021neurwin}
K.~Nakhleh, S.~Ganji, P.-C. Hsieh, I.~Hou, S.~Shakkottai \emph{et~al.},
  ``Neurwin: Neural whittle index network for restless bandits via deep rl,''
  \emph{Advances in Neural Information Processing Systems}, vol.~34, pp.
  828--839, 2021.

\bibitem{nino_2001}
J.~Ni{\~n}o-Mora, ``Restless bandits, partial conservation laws and
  indexability,'' \emph{Advances in Applied Probability}, vol.~33, no.~1, pp.
  76--98, 2001.

\bibitem{nino2007dynamic}
------, ``Dynamic priority allocation via restless bandit marginal productivity
  indices,'' \emph{Top}, vol.~15, no.~2, pp. 161--198, 2007.

\bibitem{Whittle_app2006}
K.~D. Glazebrook, D.~Ruiz-Hernandez, and C.~Kirkbride, ``Some indexable
  families of restless bandit problems,'' \emph{Advances in Applied
  Probability}, vol.~38, no.~3, pp. 643--672, 2006.

\bibitem{Gongpu2021TIT}
G.~Chen, S.~C. Liew, and Y.~Shao, ``Uncertainty-of-information scheduling: A
  restless multiarmed bandit framework,'' \emph{IEEE Transactions on
  Information Theory}, vol.~68, no.~9, pp. 6151--6173, 2022.

\bibitem{liu2010indexability}
K.~Liu and Q.~Zhao, ``Indexability of restless bandit problems and optimality
  of whittle index for dynamic multichannel access,'' \emph{IEEE Transactions
  on Information Theory}, vol.~56, no.~11, pp. 5547--5567, 2010.

\bibitem{villar2016indexability}
S.~S. Villar, ``Indexability and optimal index policies for a class of
  reinitialising restless bandits,'' \emph{Probability in the engineering and
  informational sciences}, vol.~30, no.~1, pp. 1--23, 2016.

\bibitem{TIT2023}
G.~Chen and S.~C. Liew, ``An index policy for minimizing the
  uncertainty-of-information of markov sources,'' \emph{IEEE Transactions on
  Information Theory}, vol.~70, no.~1, pp. 698--721, 2024.

\bibitem{leong2020deep}
A.~S. Leong, A.~Ramaswamy, D.~E. Quevedo, H.~Karl, and L.~Shi, ``Deep
  reinforcement learning for wireless sensor scheduling in cyber--physical
  systems,'' \emph{Automatica}, vol. 113, p. 108759, 2020.

\bibitem{wang2018deep}
S.~Wang, H.~Liu, P.~H. Gomes, and B.~Krishnamachari, ``Deep reinforcement
  learning for dynamic multichannel access in wireless networks,'' \emph{IEEE
  transactions on cognitive communications and networking}, vol.~4, no.~2, pp.
  257--265, 2018.

\bibitem{Burak2018}
B.~Demirel, A.~Ramaswamy, D.~E. Quevedo, and H.~Karl, ``Deepcas: A deep
  reinforcement learning algorithm for control-aware scheduling,'' \emph{IEEE
  Control Systems Letters}, vol.~2, no.~4, pp. 737--742, 2018.

\bibitem{wei2017deep}
T.~Wei, Y.~Wang, and Q.~Zhu, ``Deep reinforcement learning for building hvac
  control,'' in \emph{Proceedings of the 54th annual design automation
  conference 2017}, 2017, pp. 1--6.

\bibitem{xiong2022index}
G.~Xiong, X.~Qin, B.~Li, R.~Singh, and J.~Li, ``Index-aware reinforcement
  learning for adaptive video streaming at the wireless edge,'' in
  \emph{Proceedings of the Twenty-Third International Symposium on Theory,
  Algorithmic Foundations, and Protocol Design for Mobile Networks and Mobile
  Computing}, 2022, pp. 81--90.

\bibitem{xiong2024whittle}
G.~Xiong, S.~Wang, J.~Li, and R.~Singh, ``Whittle index-based q-learning for
  wireless edge caching with linear function approximation,'' \emph{IEEE/ACM
  Transactions on Networking}, 2024.

\bibitem{biswas2021learn}
A.~Biswas, G.~Aggarwal, P.~Varakantham, and M.~Tambe, ``Learn to intervene: An
  adaptive learning policy for restless bandits in application to preventive
  healthcare,'' in \emph{Proceedings of the Thirtieth International Joint
  Conference on Artificial Intelligence, {IJCAI-21}}, 8 2021, pp. 4039--4046.

\bibitem{wang2023optimistic}
K.~Wang, L.~Xu, A.~Taneja, and M.~Tambe, ``Optimistic whittle index policy:
  Online learning for restless bandits,'' in \emph{Proceedings of the AAAI
  Conference on Artificial Intelligence}, vol.~37, no.~8, 2023, pp.
  10\,131--10\,139.

\bibitem{xiong2022learning}
G.~Xiong, S.~Wang, and J.~Li, ``Learning infinite-horizon average-reward
  restless multi-action bandits via index awareness,'' \emph{Advances in Neural
  Information Processing Systems}, vol.~35, pp. 17\,911--17\,925, 2022.

\bibitem{xiong2022reinforcement}
G.~Xiong, J.~Li, and R.~Singh, ``Reinforcement learning augmented
  asymptotically optimal index policy for finite-horizon restless bandits,'' in
  \emph{Proceedings of the AAAI Conference on Artificial Intelligence},
  vol.~36, no.~8, 2022, pp. 8726--8734.

\bibitem{puterman2014markov}
M.~L. Puterman, \emph{Markov decision processes: discrete stochastic dynamic
  programming}.\hskip 1em plus 0.5em minus 0.4em\relax John Wiley \& Sons,
  2014.

\bibitem{weber1990index}
R.~R. Weber and G.~Weiss, ``On an index policy for restless bandits,''
  \emph{Journal of applied probability}, vol.~27, no.~3, pp. 637--648, 1990.

\bibitem{abounadi2001learning}
J.~Abounadi, D.~Bertsekas, and V.~S. Borkar, ``Learning algorithms for markov
  decision processes with average cost,'' \emph{SIAM Journal on Control and
  Optimization}, vol.~40, no.~3, pp. 681--698, 2001.

\bibitem{sutton2018reinforcement}
R.~S. Sutton and A.~G. Barto, \emph{Reinforcement learning: An
  introduction}.\hskip 1em plus 0.5em minus 0.4em\relax MIT press, 2018.

\bibitem{mnih2015human}
V.~Mnih, K.~Kavukcuoglu, D.~Silver, A.~A. Rusu, J.~Veness, M.~G. Bellemare,
  A.~Graves, M.~Riedmiller, A.~K. Fidjeland, G.~Ostrovski \emph{et~al.},
  ``Human-level control through deep reinforcement learning,'' \emph{nature},
  vol. 518, no. 7540, pp. 529--533, 2015.

\bibitem{tripathi2019whittle}
V.~Tripathi and E.~Modiano, ``A whittle index approach to minimizing functions
  of age of information,'' in \emph{2019 57th Annual Allerton Conference on
  Communication, Control, and Computing (Allerton)}.\hskip 1em plus 0.5em minus
  0.4em\relax IEEE, 2019, pp. 1160--1167.

\bibitem{xu2021dual}
L.~Xu, E.~Bondi, F.~Fang, A.~Perrault, K.~Wang, and M.~Tambe, ``Dual-mandate
  patrols: Multi-armed bandits for green security,'' in \emph{Proceedings of
  the AAAI Conference on Artificial Intelligence}, vol.~35, no.~17, 2021, pp.
  14\,974--14\,982.

\end{thebibliography}

\end{document}